\documentclass{article}

    \usepackage[preprint]{neurips_2021}

\usepackage[utf8]{inputenc} 
\usepackage[T1]{fontenc}    
\usepackage{hyperref}       
\usepackage{url}            
\usepackage{booktabs}       
\usepackage{amsfonts}       
\usepackage{nicefrac}       
\usepackage{microtype}      
\usepackage{xcolor}         

\usepackage{amsmath}
\usepackage{amssymb}
\usepackage{amsthm}
\usepackage{natbib}
\usepackage{colortbl}
\usepackage{booktabs} 
\usepackage{algorithm}
\usepackage{algorithmic}
\newtheorem{theorem}{Theorem}
\newtheorem{lemma}{Lemma}
\newtheorem{corollary}{Corollary}

\theoremstyle{definition}
\newtheorem{remark}{Remark}
\newtheorem{definition}{Definition}
\newtheorem{example}{Example}

\newcommand{\cD}{\mathcal{D}}

\newcommand{\prob}{\mathop{\mathrm{Prob}}\limits}

\DeclareMathOperator*{\argmin}{arg\,min}

\DeclareMathOperator*{\E}{\mathbf{E}}
\newcommand{\re}{\mathbb{R}}
\renewcommand{\epsilon}{\varepsilon}

\newcommand{\qedwhite}{\hfill \ensuremath{\Box}}

\title{On Optimal Robustness to Adversarial Corruption\\in Online Decision Problems}

\author{%
  Shinji Ito
  \\
  NEC Corporation
  \\
  \texttt{i-shinji@nec.com}
}

\begin{document}

\maketitle

\begin{abstract}
  This paper considers two fundamental sequential decision-making problems: the problem of prediction with expert advice and the multi-armed bandit problem.  We focus on stochastic regimes in which an adversary may corrupt losses, and we investigate what level of robustness can be achieved against adversarial corruptions.  The main contribution of this paper is to show that optimal robustness can be expressed by a square-root dependency on the amount of corruption.  More precisely, we show that two classes of algorithms, anytime Hedge with decreasing learning rate and algorithms with second-order regret bounds, achieve $O( \frac{\log N}{\Delta} + \sqrt{ \frac{C \log N }{\Delta} } )$-regret, where $N, \Delta$, and $C$ represent the number of experts, the gap parameter, and the corruption level, respectively.  We further provide a matching lower bound, which means that this regret bound is tight up to a constant factor. For the multi-armed bandit problem, we also provide a nearly tight lower bound up to a logarithmic factor.
\end{abstract}

\section{Introduction}
In this work,
we consider
two fundamental sequential decision-making problems,
the problem of prediction with expert advice ({expert problem})
and the {multi-armed bandit (MAB) problem}.
In both problems,
a player chooses probability vectors $p_t$ over a 
given action set $[N]=\{ 1, 2, \ldots, N \}$ in a sequential manner.
More precisely,
in each round $t$,
a player chooses a probability vector $p_t \in [0, 1]^N$ over the action set, and then
an environment chooses a loss vector $\ell_t \in [0, 1]^N$.
After the player chooses $p_t$,
the player observes $\ell_t$ in the expert problem.
In an MAB problem,
the player picks action $i_t \in [N]$ following $p_t$ and then
observe $\ell_{ti_t}$.
The goal of the player is to minimize the (pseudo-) regret $\bar{R}_T$ defined as
\begin{align}
  \label{eq:defRT}
  {R}_{T i^*}
  =
  \sum_{t=1}^T
  \ell_t^\top p_t
  -
  \sum_{t=1}^T
  \ell_{ti^*} ,
  \quad
  \bar{R}_{T i^*}
  =
  \E \left[
    {R}_{T i^*}
  \right],
  \quad
  \bar{R}_{T}
  =
  \max_{i^* \in [N]} \bar{R}_{Ti^*} .
\end{align}

For such decision-making problems,
two main types of environments have been studied:
stochastic environments and an adversarial environments.
In stochastic environments,
the loss vectors are assumed to follow an unknown distribution, i.i.d.~for all rounds.
It is known that the difficulty of the problems can be characterized by the
\textit{suboptimality gap} parameter $\Delta > 0$,
which denotes the minimum gap between the expected losses for the optimal action and for suboptimal actions. 
Given the parameter $\Delta$,
mini-max optimal regret bounds can be expressed as $\Theta ( \frac{\log N}{\Delta}  )$ in the expert problem \citep{degenne2016anytime,mourtada2019optimality} 
and $\Theta( \frac{N \log T}{\Delta})$ in the MAB problem \citep{auer2002finite,lai1985asymptotically,lai1987adaptive}.
In contrast to the stochastic model,
the adversarial model does not assume any generative models for loss vector,
but the loss at each round may behave adversarially depending on the choices of the player up until that round.
The mini-max optimal regret bounds for the adversarial model are
$\Theta ( \sqrt{T \log N} )$ in the expert problem \citep{cesa2006prediction,freund1997decision} 
and $\Theta( \sqrt{TN} )$ in the MAB problem \citep{audibert2009minimax,auer2002nonstochastic}.

As can be seen in the regret bounds,
achievable performance differs greatly between the stochastic and adversarial regimes,
which implies that
the choice of the models and algorithms will matter in many practical applications.
One promising solution to this challenge is to develop
\textit{best-of-both-worlds (BOBW)} algorithms,
which perform (nearly) optimally in both stochastic and adversarial regimes.
For the expert problem,
\citet{gaillard2014second} provide an algorithm with a BOBW property,
and \citet{mourtada2019optimality} have shown that the well-known Hedge algorithm with decreasing learning rate (decreasing Hedge)
enjoys a BOBW property as well.
For the MAB problem,
the Tsallis-INF algorithm by \citet{zimmert2021tsallis} has a BOBW property,
i.e.,
achieves $O(\frac{K \log T}{\Delta})$-regret in the stochastic regime and
$O(\sqrt{KT})$-regret in the adversarial regime.
One limitation of BOBW guarantees is,
however,
that it does not necessarily provide nontrivial regret bounds for a situation in which the stochastic and the adversarial regimes are \textit{mixed},
i.e.,
an intermediate situation.

To overcome this BOBW-property limitation,
our work focuses on an intermediate (and comprehensive) regime between the stochastic and adversarial settings.
More precisely,
we consider the \textit{adversarial regime with a self-bounding constraint} introduced by \citet{zimmert2021tsallis}.
As shown by them,
this regime includes the \textit{stochastic regime with adversarial corruptions} \citep{lykouris2018stochastic,amir2020prediction} as a special case,
in which an adversary modifies the i.i.d.~losses to the extent that the total amount of changes does not exceed $C$,
which is an unknown parameter referred to as the \textit{corruption level}.
For the expert problems in the stochastic regime with adversarial corruptions,
\citet{amir2020prediction} have shown that the decreasing Hedge algorithm achieves $O(\frac{\log{N}}{\Delta} + C)$-regret.
For the MAB problem,
\citet{zimmert2021tsallis} have shown that Tsallis-INF achieves
$O(\frac{N \log T}{\Delta} + \sqrt{ \frac{C N \log T}{\Delta}})$-regret in adversarial regimes with self-bounding constraints.
To be more precise,
they have proved a regret upper bound of
$O( \sum_{i \neq i^*} \frac{ \log T}{\Delta_i} + \sqrt{C \sum_{i \neq i^*} \frac{\log T}{\Delta_i}}  )$,
where $i^* \in [N]$ is the optimal action and $\Delta_i$ represents the suboptimality gap for each action $i \in [N] \setminus \{ i^* \}$.
In addition to this,
\citet{masoudian2021improved} have improved the analysis to obtain a refined regret bound of
$O\left( \left( \sum_{i\neq i^*} \frac{1}{\Delta_i} \right) \log_+ \left( \frac{(K-1)T}{( \sum_{i\neq i^*} \frac{1}{\Delta_i} )^2 } \right) 
+ \sqrt{C \left( \sum_{i \neq i^*} \frac{1}{\Delta_i} \right) \log_+ \left( \frac{(K-1)T}{C \sum_{i\neq i^*} \frac{1}{\Delta_i}} \right) } \right)$,
where $\log_+ (x) = \max \left\{ 1, \log x \right\}$.
\citet{ito2021parameter} has shown that similar regret bounds hold even when there are multiple optimal actions,
i.e.,
even if the number of actions $i$ with $\Delta_i = 0$ is greater than $1$.
In such cases,
the terms of $\sum_{i\neq i^*} 1/\Delta_i$ in regret bounds are replaced with $\sum_{i : \Delta_i > 0} \frac{1}{\Delta_i}$.

\begin{table}
  \caption{Regret bounds in stochastic regimes with adversarial corruptions}
  \label{table:regretbound}
  \centering
  \begin{tabular}{lll}
    \toprule
    Problem setting     & Upper bound     & Lower bound \\
    \midrule
    Expert problem & $O \left( \frac{\log N}{\Delta} + C \right) $\footnotemark \citep{amir2020prediction}  & $\Omega \left( \frac{\log N}{\Delta} + \sqrt{ \frac{C \log N}{\Delta}} \right)$     \\
    & $O \left( \frac{\log N}{\Delta} + \sqrt{ \frac{C \log N}{\Delta}} \right)$ [Theorems~\ref{thm:decreasingHedge}, \ref{thm:second}] \quad \quad  &    [Theorem~\ref{thm:lowerexpert}] \\
    \midrule 
    Multi-armed bandit  & $O\left( \frac{N \log T}{\Delta} + \sqrt{\frac{C N \log T}{\Delta}} \right)$  & $\Omega \left( \frac{N}{\Delta} + \sqrt{\frac{ C N }{\Delta}} \right)$      \\
    & \citep{zimmert2021tsallis} &   [Theorem~\ref{thm:lowerMAB}] \\
    \bottomrule
  \end{tabular}
\end{table}
\footnotetext{
  Note that 
  \citet{amir2020prediction}
  adopt a different definition of regret than in this paper.
  Details and notes for comparison are discussed in Remark~\ref{rem:comparison}.
}

The contributions of this work are summarized in Table~\ref{table:regretbound},
with existing results.
As shown in Theorems~\ref{thm:decreasingHedge} and \ref{thm:second},
this paper provides an improved regret upper bound
of $O ( \frac{\log N}{\Delta} + \sqrt{ \frac{C \log N}{\Delta}} )$ 
for the expert problem in the adversarial regime with self-bounding constraints.
This regret upper bound is tight up to a constant factor.
In fact,
we provide a matching lower bound in Theorem~\ref{thm:lowerexpert}.
In addition to this,
we show an $\Omega ( \frac{N}{\Delta} + \sqrt{\frac{CN}{\Delta}} )$-lower bound
for MAB,
which implies that Tsallis-INF by \citet{zimmert2021tsallis} achieves a nearly optimal regret bound up to an $O(\log T)$ factor
in the adversarial regime with self-bounding constraints.


The regret bounds in Theorems~\ref{thm:decreasingHedge} and \ref{thm:second} are smaller
than the regret bound shown by \citet{amir2020prediction} for the stochastic regime with adversarial corruptions,
especially when
$C = \Omega(\frac{\log N}{\Delta})$,
and they can be applied to more general problem settings of the adversarial regime with self-bounding constraints.
Note here that this study and their study consider slightly different definitions of regret.
More precisely,
they define regret using losses \textit{without} corruptions,
while this study uses losses \textit{after} corruption to define regret.
In practice,
appropriate definitions would vary depending on situations.
For example,
if each expert's prediction is itself corrupted,
the latter definition is suitable.
In contrast,
if only the observation of the player is corrupted,
the former definition seems appropriate.
However,
even after taking this difference in definitions into account,
we can see that the regret bound in our work is,
in a sense,
stronger than theirs,
as is discussed in Remark~\ref{rem:comparison} of this paper.
In particular,
we would like to emphasize that the new bound of $O(\frac{\log N}{\Delta} + \sqrt{\frac{C \log N}{\Delta}})$
provides the first theoretical evidence implying that the corresponding algorithms are more robust
than the naive Follow-the-Leader algorithm,
against adversarial corruptions.
On the other hand,
we also note that
the regret bound by \citet{amir2020prediction} is tight as long as the former regret definition is used.

This work shows the tight regret upper bounds for two types of known algorithms.
The first, (Theorem~\ref{thm:decreasingHedge}), is the decreasing Hedge algorithm,
which has been analyzed by \citet{amir2020prediction} and \citet{mourtada2019optimality} as well.
The second (Theorem~\ref{thm:second}) is algorithms with \textit{second-order regret bounds} \citep{cesa2007improved,gaillard2014second,hazan2010extracting,steinhardt2014adaptivity,luo2015achieving}.
It is worth mentioning that \citet{gaillard2014second} have shown that a kind of second-order regret bounds
imply $O(\frac{\log N}{\Delta})$-regret in the stochastic regime.
Theorem~\ref{thm:second} in this work extends their analysis to a broader setting of the adversarial regime with self-bounding constraints.
In the proof of Theorems~\ref{thm:decreasingHedge} and \ref{thm:second},
we follow a proof technique given by \citet{zimmert2021tsallis} to exploit self-bounding constraints.

To show regret lower bounds in Theorems~\ref{thm:lowerexpert} and \ref{thm:lowerMAB},
we construct specific environments with corruptions
that provide insight into effective attacks which would make learning fail.
Our approach to corruption is to modify the losses so that the optimality gaps reduce.
This approach achieves a mini-max lower bound in the expert problem (Theorem~\ref{thm:lowerexpert})
and a nearly-tight lower bound in MAB up to a logarithmic factor in $T$ (Theorem~\ref{thm:lowerMAB}).
We conjecture that there is room for improvement in this lower bound for MAB under assumptions on consistent policies~\citep{lai1985asymptotically},
and that the upper bound by \citet{zimmert2021tsallis} is tight up to a constant factor.

\section{Related work}
In the context of the expert problem,
studies on stochastic settings seem to be more limited than those on the adversarial setting.
\citet{de2014follow} have focused on the fact that the Follow-the-Leader (FTL) algorithm works well for a stochastic setting,
and they have provided an algorithm that combines FTL and Hedge algorithms to achieve the best of both worlds.
\citet{gaillard2014second} have provided an algorithm with a second-order regret bound depending on
$V_{Ti^*} = \sum_{t=1}^T (\ell_t^\top p_t - \ell_{ti^*})^2$ in place of $T$
and have shown that such an algorithm achieves $O(\frac{\log N}{\Delta})$-regret in the stochastic regime.
\citet{mourtada2019optimality} have shown that a simpler Hedge algorithm with decreasing learning rates of
$\eta_t = \Theta(\sqrt{\frac{\log N}{t}})$
enjoys a tight regret bound in the stochastic regime as well.
This simple decreasing Hedge algorithm has been shown
by \citet{amir2020prediction}
to achieve $O(\frac{\log N}{\Delta} + C)$-regret in the stochastic regime with adversarial corruptions.
For online linear optimization,
a generalization of the expert problem,
\citet{huang2016following} have shown that FTL achieves smaller regret in the stochastic setting
and provided best-of-both-worlds algorithms via techniques reported by \citet{sani2014exploiting}.

For MAB,
there are a number of
studies on best-of-both-worlds algorithms \citep{bubeck2012best,zimmert2021tsallis,seldin2014one,seldin2017improved,pogodin2020first,auer2016algorithm,wei2018more,zimmert2019beating,lee2021achieving,ito2021parameter}.
Among these,
studies by \citet{wei2018more,zimmert2021tsallis,zimmert2019beating} are closely related to this work.
In their studies,
gap-dependent regret bounds in the stochastic regime are derived via $\{ p_t \}$-dependent regret bounds in the adversarial regime,
similarly to that seen in this work and in previous studies by \citet{gaillard2014second,amir2020prediction}.

Studies on online optimization algorithms robust to adversarial corruptions
has been extended to a variety of models,
including those for the multi-armed bandit \citep{lykouris2018stochastic,gupta2019better,zimmert2021tsallis,hajiesmaili2020adversarial},
Gaussian process bandits \citep{bogunovic2020corruption},
Markov decision processes \citep{lykouris2019corruption},
the problem of prediction with expert advice \citep{amir2020prediction},
online linear optimization \citep{li2019stochastic},
and linear bandits \citep{bogunovic2021stochastic,lee2021achieving}.
There can be found the literature on effective attacks to bandit algorithms \citep{jun2018adversarial,liu2019data} as well.

As summarized by \citet{hajiesmaili2020adversarial},
there can be found studies on two different models of adversarial corruptions:
the oblivious corruption model and the targeted corruption model.
In the former (e.g., in studies by \citet{lykouris2018stochastic,gupta2019better,bogunovic2020corruption}) the attacker may corrupt the losses $\ell_t$ after observing $( \ell_t, p_t )$ without knowing 
the chosen action $i_t$ while,
in the latter (e.g., in studies by \citet{jun2018adversarial,hajiesmaili2020adversarial,liu2019data,bogunovic2021stochastic}),
the attacker can choose corruption depending on $(\ell_t, p_t, i_t)$.
We discuss the differences between these models in Section~\ref{sec:setting}.
This work mainly focuses on the oblivious corruption model for MAB problems.
It is worth mentioning that Tsallis-INF \citep{zimmert2021tsallis} works well in the oblivious corruption models,
as is shown in Table~\ref{table:regretbound},
as well as achieving best-of-both-worlds.

\section{Problem setting}
\label{sec:setting}
A player is given $N$,
the number of actions.
In each round $t = 1, 2, \ldots$
the player chooses a probability vector $p_t = (p_{t1}, p_{t2}, \ldots, p_{tN})^\top \in \{ p \in [0, 1]^N \mid \| p \|_1 = 1 \} $,
and then the environment chooses a loss vector $\ell_t = ( \ell_{t1} , \ell_{t2}, \ldots, \ell_{tN} )^\top \in [0, 1]^N$.
In the expert problem,
the player can observe all entries of $\ell_t$ after outputting $p_t$.
By way of contrast,
in MAB problem,
the player picks $i_t$ w.r.t.~$p_t$,
i.e.,
choose $i_t$ so that $\prob [i_t = i | p_t] = p_{ti}$,
and then observes $\ell_{ti_t}$.
The performance of the player is measured by means of the regret defined in \eqref{eq:defRT}.

Note that
in MAB problems
we have
\begin{align}
  \label{eq:MABreg}
  \E
  \left[
    \sum_{t=1}^{T}
    \left(
      \ell_{ti_t}
      -
      \ell_{ti^*}
    \right)
  \right]
  =
  \E
  \left[
    \sum_{t=1}^{T}
    \left(
      \ell_{t}^\top p_t
      -
      \ell_{ti^*}
    \right)
  \right]
  =
  \bar{R}_{Ti^*}
\end{align}
under the assumption that $\ell_t$ is independent of $i_t$ given $p_t$.


This paper focuses on environments in the following regime:
\begin{definition}[Adversarial regime with a self-bounding constraint \citep{zimmert2021tsallis}]
  \label{def:ARSC}
  We say that the environment is in an \textit{adversarial regime with a} $(i^*, \Delta, C, T)$ \textit{self-bounding constraint} if
  \begin{align}
    \label{eq:defARSC}
    \bar{R}_{Ti^*}
    =
    \E \left[
      \sum_{t=1}^{T} ( \ell_{t}^\top p_t - \ell_{ti^*} )
    \right]
    \geq
    \Delta 
    \cdot
    \E \left[
      \sum_{t=1}^{T} (1 - p_{ti^*}) 
    \right]
    - C
  \end{align}
  holds for any algorithms,
  where $\Delta \in [0, 1]$ and $C \geq 0$.
\end{definition}
In this paper,
we deal with the situation in which
the player is \textit{not} given parameters $(i^*, \Delta, C, T)$.
We note that the environment in Definition~\ref{def:ARSC} includes the adversarial setting as a special case
since \eqref{eq:defARSC} holds for any $\Delta \in [0, 1]$ and $i^*$ if $C \geq 2T$.

The regime defined in Definition~\ref{def:ARSC} includes the following examples:
\begin{example}[Stochastic regime]
  Suppose $\ell_t \in [0, 1]^N$ follows an unknown distribution $\cD$ over  i.i.d.~for $t \in \{  1, 2,  \ldots \}$.
  Denote $\mu = \E_{\ell \sim \cD } [ \ell ]$,
  and let $i^* \in \argmin_{i \in [N]} \mu_i $
  and $\Delta = \min_{i \in [N] \setminus \{ i^* \}} ( \mu_i - \mu_{i^*} )$.
  Then the environment is in the adversarial regime with a self-bounding constraint \eqref{eq:defARSC} with $C = 0$.
  Note here that $\Delta > 0$ implies that the optimal action is \textit{unique},
  i.e.,
  $\mu_i > m_{i^*}$ holds for any action $i \in [N] \setminus \{ i^* \}$.
\end{example}
\begin{example}[Stochastic regime with adversarial corruptions]
  \label{exa:SRAC}
  Suppose $\ell_t \in [0, 1]^N$ is given as follows:
  (i) a temporary loss $\ell_t' \in [0, 1]^N$ is generated from an unknown distribution $\cD$ (i.i.d. for $t$)
  (ii) an adversary corrupts $\ell_t'$ after observing $p_t$ to determine $\ell_t$ subject to the constraint of
  $\sum_{t=1}^T \| \E[ \ell_t ] - \E [ \ell_t' ] \|_{\infty} \leq C$.
  As shown in \citep{zimmert2021tsallis},
  this regime satisfies \eqref{eq:defARSC},
  i.e.,
  is a special case of the adversarial regime with a self-bounding constraint.
\end{example}
\begin{remark}
  \label{rem:comparison}
  For the stochastic regime with adversarial corruptions,
  different regret notions can be found in the literature.
  An alternative to the definition in \eqref{eq:defRT} is
  the regret w.r.t.~the losses without corruptions,
  i.e.,
  $R'_{Ti^*} = \sum_{t=1}^T \left( \ell_{t}'^\top p_t - \ell_{ti^*}'\right) $,
  ($\bar{R}'_{Ti^*}$, and $\bar{R}'_T$ can also be defined in a similar way).
  In general,
  which metric will be appropriate depends on the situation of the application.
  For example,
  in the case of prediction with expert advice,
  if each expert's prediction is itself corrupted,
  the player's performance should be evaluated in terms of the regret $\bar{R}_T$,
  as the consequential prediction performance is determined by the losses $\ell_t$ after corruptions,
  not by $\ell_t'$.
  In contrast,
  if only the observation of the player is corrupted,
  the performance should be evaluated in terms of $\bar{R}_{T}'$.

  We can easily see that $| \bar{R}'_{Ti^*} - \bar{R}_{Ti^*} | \leq 2 C $.
  \citet{amir2020prediction} have shown
  a regret bound of $\bar{R}'_{T} = O \left( \frac{\log N}{\Delta} + C \right)$,
  which immediately implies $\bar{R}_{T} = O\left( \frac{\log N}{\Delta} + C \right)$.
  Similarly,
  a regret bound of 
  $\bar{R}_T = O\left( \frac{\log N}{\Delta} + \sqrt{\frac{ C \log N}{\Delta}} \right)$
  immediately implies
  $\bar{R}'_{T} = O \left( \frac{\log N}{\Delta} + C \right)$.
  In fact,
  from the AM-DM inequality,
  we have
  \begin{align*}
    \frac{\log N}{\Delta} + \sqrt{\frac{ C \log N}{\Delta}}
    \leq
    \frac{\log N}{\Delta} + 
    \frac{1}{2}\left(
    C + \frac{\log N}{\Delta}
    \right)
    =
    O\left(
    C + \frac{\log N}{\Delta}
    \right).
  \end{align*}
  We here note that the former bound of
  $\bar{R}_T = O\left( \frac{\log N}{\Delta} + \sqrt{\frac{ C \log N}{\Delta}} \right)$
  is properly stronger than the latter of
  $\bar{R}'_{T} = O \left( \frac{\log N}{\Delta} + C \right)$,
  as the latter does not necessarily implies the former.
\end{remark}
\begin{remark}
  In MAB,
  a \textit{targeted corruption model} has been considered
  to be a variant of the model in Example~\ref{exa:SRAC}.
  In this model,
  the adversary corrupts the losses after observing $i_t$.
  In this case,
  the loss $\ell_t$ after corruptions and $i_t$ are dependent given $p_t$,
  and hence \eqref{eq:MABreg} does not always hold.
\end{remark}

\section{Regret upper bound}
\subsection{Known regret bounds for adversarial regimes by hedge algorithms}
The Hedge algorithm \citep{freund1997decision}
(also called the multiplicative weight update \citep{arora2012multiplicative} or the weighted majority forecaster \citep{littlestone1994weighted})
is known to be a mini-max optimal algorithm for the expert problem.
In the Hedge algorithm,
the probability vector $p_t$ is defined as follows:
\begin{align}
  \label{eq:defMWU}
  w_{ti} = \exp \Big( - \eta_t \sum_{j=1}^{t-1} \ell_{ji} \Big),
  \quad
  p_{t} = \frac{w_t}{\| w_t \|_1},
\end{align}
where $\eta_t > 0$ are learning rate parameters.
If $p_t$ is given by \eqref{eq:defMWU},
the regret is bounded as follows:
\begin{lemma}
  \label{lem:boundMWU}
  If $\{ p_t \}_{t=1}^T$ is given by \eqref{eq:defMWU} with decreasing learning rates (i.e., $\eta_t \geq \eta_{t+1}$ for all $t$),
  for any $\{ \ell_t \}_{t=1}^T$ and $i^* \in [N]$,
  the regret is bounded as
  \begin{align}
    \label{eq:boundMWU0}
    R_{Ti^*}
    \leq
    \frac{\log N}{\eta_1}
    +
    \sum_{t=1}^T
    \left(
      \frac{1}{\eta_t}
      \sum_{i=1}^N p_{ti} g \left( \eta_t ( -  \ell_{ti} + \alpha_t ) \right)
      +
      \left(
        \frac{1}{\eta_{t+1}}
        -
        \frac{1}{\eta_{t}}
      \right)
      H (p_{t+1})
    \right) ,
  \end{align}
  for any $\{ \alpha_t \}_{t=1}^{T} \subseteq \re$,
  where $g$ and $H$ are defined as
  \begin{align}
    g(x) = \exp(x) - x - 1,
    \quad
    H (p) = \sum_{i=1}^{N} p_i \log \frac{1}{p_i} .
  \end{align}
\end{lemma}
From this lemma,
using $g(x) \approx x^2 / 2$ and $H(p) \leq \log N$,
we obtain the following regret bounds for adversarial settings.
\begin{theorem}[Theorem 2.3 in \citep{cesa2006prediction}]
  \label{thm:decreasingHedgeAdv}
  If $p_t$ is given by \eqref{eq:defMWU} with $\eta_t = \sqrt{\frac{8 \log N}{t} }$,
  for any $T$, $i^* \in [N]$ and $\{ \ell_t \}_{t=1}^T \subseteq [0, 1]^N$,
  the regret is bounded as
  \begin{align}
    R_{Ti^*}
    \leq
    \sqrt{2 T \log N}
    +
    \log N / 8 .
  \end{align}
\end{theorem}
Hedge algorithms with decreasing learning rates $\eta_t = \Theta ( \sqrt{\frac{\log N}{t}} )$,
as in Theorem~\ref{thm:decreasingHedgeAdv},
are referred to as decreasing Hedge, e.g., in \citep{mourtada2019optimality}.
Such algorithms are shown 
by \citet{mourtada2019optimality}
to achieve $O(\sqrt{\frac{\log N}{\Delta}})$-regret in stochastic regimes,
and are also shown,
by \citet{amir2020prediction},
to achieve $O ( \sqrt{\frac{\log N}{\Delta}} + C )$-regret 
in stochastic regimes with adversarial corruptions.

Besides such worst-case regret bounds as found in Theorem~\ref{thm:decreasingHedgeAdv},
various data-dependent regret bounds have been developed (see, e.g., \citep{steinhardt2014adaptivity}).
One remarkable example is that of the \textit{second-order} bounds by \citet{cesa2007improved},
which depend on parameters $V_T$ defined as follows:
\begin{align}
  \label{eq:defVT}
  v_t = \sum_{i=1}^N p_{ti} ( \ell_{ti} - \ell_t^\top p_t )^2,
  \quad
  V_T = \sum_{t=1}^T v_t
\end{align}
A regret bound depending on $V_T$ rather than $T$
can be achieved by the following adaptive learning rates:
\begin{theorem}[Theorem 5 in \citep{cesa2007improved}]
  \label{thm:secondAdv}
  If $p_t$ is given by \eqref{eq:defMWU} with,
    $
    \eta_t = \min \left\{ 1, \sqrt{\frac{2 (\sqrt{2} - 1) \log N}{(e - 2) V_{t-1}}} \right\},
    $
  the regret is bounded as
  \begin{align}
    R_{Ti^*}
    \leq
    4
    \sqrt{V_{T} \log N}
    +
    2 \log N
    +
    1 / 2
  \end{align}
  for any $T$, $i^*$ and $\{ \ell_{t} \}_{t = 1}^T \subseteq [0, 1]^N$.
\end{theorem}
As $V_T \leq T$ follows from the definition \eqref{eq:defVT},
the regret bound in \eqref{thm:secondAdv} includes the worst-case regret bound of $R_{Ti^*} = O( \sqrt{ T \log N } )$.
Further,
as shown in Corollary 3 of \citep{cesa2007improved},
the bound in Theorem~\ref{thm:secondAdv} implies
$R_{Ti^*} = O ( \sqrt{ \frac{ L_{Ti^*} ( T - L_{Ti^*} ) }{T} \log N }  )$,
where $L_{Ti^*} = \sum_{t=1}^T \ell_{ti^*} $.
This means that the regret will be improved if the cumulative loss $L_{Ti^*}$
for optimal action $i^*$ is small or is close to $T$.

\subsection{Refined regret bound for decreasing Hedge}
This subsection shows that
the algorithm described in Theorem~\ref{thm:decreasingHedgeAdv} enjoys the
following regret bound as well:
\begin{theorem}
  \label{thm:decreasingHedge}
  If $p_t$ is given by \eqref{eq:defMWU} with $\eta_t = \sqrt{\frac{8 \log N}{t} }$,
  we have
  \begin{align}
    \bar{R}_{T i^*}
    \leq
    33 +
    \frac{100 \log N }{\Delta}
    +
    10
    \sqrt{
      \frac{C \log N }{\Delta}
    }
  \end{align}
  under the assumption that \eqref{eq:defARSC} holds.
\end{theorem}
\begin{proof}
  Using the fact that
  $g(x) \leq \frac{(e - 1)x^2}{2} $ for $x \leq 1$
  and 
  $H(p) \leq (1-p_{i^*}) (1 + \log N - \log (1-p_{i^*}))$,
  from Lemma~\ref{lem:boundMWU},
  we obtain
  \begin{align}
    \label{eq:decreasing1}
    R_{Ti^*}
    \leq
    32 \log N
    +
    \sqrt{8 \log N}
    \sum_{t=1}^{T}
    \frac{1 - p_{ti^*}}{\sqrt{t}}
    \left(
      1
      +
      \frac{1}{16 \log N } 
      \log \frac{1}{1 - p_{ti^*}}
    \right) .
  \end{align}
  A complete proof for \eqref{eq:decreasing1} can be found in the appendix.
  From \eqref{eq:defARSC} and \eqref{eq:decreasing1},
  for any $\lambda > 0$,
  we have
  \begin{align*}
    &
    \bar{R}_{Ti^*}
    =
    (1 + \lambda)
    \bar{R}_{Ti^*}
    -
    \lambda
    \bar{R}_{Ti^*}
    \\
    &
    \leq
    \E \left[
      (1+ \lambda)
      \left(
      32 \log N
      +
      \sqrt{8 \log N}
      \sum_{t=1}^{T}
      \frac{1 - p_{ti^*}}{\sqrt{t}}
      \left(
        1
        -
        \frac{\log (1 - p_{ti^*})}{16 \log N } 
      \right) 
      \right)
      - \lambda
      \left(
        \Delta \sum_{t=1}^T (1 - p_{ti^*})
        - C
      \right)
    \right]
    \\
    &
    \leq
    32
    (1+\lambda) \log N
    +
    \lambda C
    +
    \frac{ 1 + \lambda }{ \sqrt{32 \log N} }
    \E\left[
      \sum_{t=1}^T
      \frac{1 - p_{ti^*}}{\sqrt{t}}
      \left(
        16 \log N - \frac{\lambda \Delta \sqrt{32 t \log N} }{1+\lambda}
        -
        \log ({1 - p_{ti^*}})
      \right)
      \right].
  \end{align*}
  To bound the values of the expectation,
  we use the following inequality
  \begin{align}
    \label{eq:decreasing2}
    \sum_{t=1}^T \frac{x_t}{\sqrt{t}}
    \left(a - b \sqrt{t} - \log x_t \right)
    \leq
    \frac{2 a^2 + 1}{b}
    +
    b
  \end{align}
  which holds for any $a, b > 0$, $T$ and $\{ x_t \}_{t=1}^T \subseteq (0, 1)$.
  A proof of \eqref{eq:decreasing2} is given in the appendix.
  Combining the above two displayed inequalities
  with $a = 16 \log N$
  and $b = \frac{\lambda \Delta \sqrt{{32} \log N}}{1 + \lambda}$
  we obtain
  \begin{align*}
    \bar{R}_{Ti^*}
    &
    \leq
    32 (1+\lambda) + \lambda C
    + \frac{ (1+\lambda)^2 ( 16 \log N + 1)}{\lambda \Delta}
    + \Delta
    \\
    &
    =
    33
    +
    \frac{36 \log N }{\Delta}
    +
    \lambda \left(
      32 + C
      +
      \frac{18 \log N }{\Delta}
    \right)
    +
    \frac{1}{ \lambda } 
    \frac{18 \log N  }{\Delta} .
  \end{align*}
  By choosing
  $\lambda = \sqrt{\frac{16 \log N}{\Delta(32 + C) + 16 \log N}}$,
  we obtain
  \begin{align*}
    \bar{R}_{Ti^*}
    &
    \leq
    33 +
    \frac{36 \log N }{\Delta}
    +
    2
    \sqrt{
      \left(
      32 + C
        +
        \frac{18 \log N }{\Delta}
      \right)
      \frac{18 \log N }{\Delta}
    }
    \\
    &
    \leq
    33 +
    \frac{100 \log N }{\Delta}
    +
    10
    \sqrt{
      \frac{C \log N }{\Delta}
    },
  \end{align*}
  where the second inequality follows from $\sqrt{x + y} \leq \sqrt{x} + \sqrt{y}$
  for $x, y \geq 0$.
\end{proof}
Combining Theorems~\ref{thm:decreasingHedgeAdv} and \ref{thm:decreasingHedge},
we can see that
the decreasing Hedge with $\eta_t = \sqrt{\frac{8 \log N}{t}}$
achieves
$\bar{R}_{Ti^*} = O ( \min\{ \frac{\log N}{\Delta} + \sqrt{\frac{C \log N}{\Delta}}, \sqrt{T \log N}  \} )$
in the adversarial regive with self-bounding constraints.
This bound will be shown 
to be tight up to a constant factor
in Section~\ref{sec:lower}.

\subsection{Refined regret bound for adaptive Hedge}
In this subsection,
we show that a second-order regret bound as seen in Theorem~\ref{thm:secondAdv}
implies tight gap-dependent regret bounds 
in the adversarial regime with a self-bounding constraint.

We start from the observation that $v_t$ defined in \eqref{eq:defVT} satisfies
$v_t \leq (1 - p_{ti^*})$ for any $i^*$.
In fact,
we have
$v_t \leq \sum_{i=1}^N p_{ti}( \ell_{ti} - \alpha )^2$
for any $\alpha \in \re$ as the right-hand side is minimized when $\alpha = \ell_{t}^\top p_t$,
from which it follows that
\begin{align}
  v_t \leq \sum_{i = 1}^N p_{ti}(\ell_{ti} - \ell_{ti^*} )^2
  =
  \sum_{i \in [N] \setminus \{ i^* \} } p_{ti} (\ell_{ti} - \ell_{ti^*})^2
  \leq
  \sum_{i \in [N] \setminus \{ i^* \} } p_{ti}
  = 1 - p_{ti^*} .
\end{align}
Hence,
the regret bound in Theorem~\ref{thm:secondAdv}
implies
\begin{align}
  \label{eq:bound1p}
R_{Ti^*}
\leq
4 \sqrt{\log N \sum_{t=1}^T (1 - p_{ti^*})}
+ 2 \log N +  1/2.
\end{align}
Such a regret bound depending on $
 \sum_{t=1}^T (1 - p_{ti^*})
$
leads to a tight gap-dependent regret bound,
as shown in the following theorem:
\begin{theorem}
  \label{thm:second}
  Suppose that the regret is bounded as
  \begin{align}
    \label{eq:aspsecond}
    R_{Ti^*}
    \leq
    \sqrt{A \sum_{t=1}^T (1 - p_{ti^*})} + B.
  \end{align}
  Then,
  under the condition of \eqref{eq:defARSC},
  the pseudo-regret is bounded as
  \begin{align}
    \label{eq:resultsecond}
    \bar{R}_{Ti^*}
    \leq
    \frac{A}{\Delta} + B
    +
    \sqrt{
      \frac{A (B + C) }{2 \Delta} 
    }.
  \end{align}
\end{theorem}
\begin{proof}
  From \eqref{eq:aspsecond} and 
  \eqref{eq:defARSC},
  for any $\lambda > 0$
  we have
  \begin{align*}
    \bar{R}_{Ti^*}
    &
    =
    (1+\lambda)\bar{R}_{Ti^*}
    -
    \lambda \bar{R}_{Ti^*}
    \\
    &
    \leq
    \E
    \left[
    (1+\lambda)
    \left(
      \sqrt{A  \sum_{t=1}^T (1 - p_{ti^*})} + B
    \right)
    -
    \lambda
    \left(
      \Delta
      \sum_{t=1}^T (1 - p_{ti^*})
      - C
    \right)
    \right]
    \\
    &
    \leq
    \frac{A (1+\lambda)^2 }{4 \lambda \Delta}
    +
    (1 + \lambda) B + \lambda C
    =
    \frac{A  }{2 \Delta} + B
    +
    \frac{1}{\lambda}
    \frac{A  }{4 \Delta} 
    +
    \lambda \left(
      \frac{A }{4 \Delta} 
      +
      B + C 
    \right),
  \end{align*}
  where the second inequality follows from
  $a \sqrt{x} - b x = - b (\sqrt{x} - \frac{a}{2b})^2 + \frac{a^2}{4 b} \leq \frac{a^2}{4b}$
  for $a > 0$, $b \in \re$ and $x \geq 0$.
  By choosing
  $\lambda = \sqrt{\frac{A}{A + 4 \Delta (B + C)}}$,
  we obtain
  \begin{align*}
    \bar{R}_{Ti^*}
    \leq
    \frac{A }{2 \Delta} + B
    +
    \sqrt{
      \left(
      \frac{A }{2 \Delta} 
      \right)^2
      +
      \frac{A (B + C)  }{2 \Delta} 
    }
    \leq
    \frac{A }{\Delta} + B
    +
    \sqrt{
      \frac{A (B + C) }{2 \Delta} 
    },
  \end{align*}
  where the second inequality follows from $\sqrt{x + y} \leq \sqrt{x} + \sqrt{y}$
  for $x, y \geq 0$.
\end{proof}
Combining this theorem and \eqref{eq:bound1p},
we obtain the following regret bound
for the algorithm described in Theorem~\ref{thm:secondAdv}.
\begin{corollary}
  If $p_t$ is chosen by \eqref{eq:defMWU}
  with $\eta_t = \min\{ 1, \sqrt{\frac{2 (\sqrt{2} - 1) \log N}{(e - 2) V_{t-1}}} \}$,
  under the condition of \eqref{eq:defARSC},
  the pseudo regret is bounded as
  $
  \bar{R}_{Ti^*}
  \leq
  16 \frac{\log N}{\Delta}
  + 4 \sqrt{\frac{ ( 3 \log N + C ) \log N}{\Delta}}
  + 3 \log N
  $.
\end{corollary}
Theorem~\ref{thm:second} can be applied to algorithms other than the one in Theorem~\ref{thm:secondAdv}.
One example is an algorithm by \citet{gaillard2014second}.
In Corollary 8 of their paper,
a regret bound of 
$R_{Ti^*} \leq C_1 \sqrt{\log N \sum_{t=1}(\ell_t^\top p_t - \ell_{ti^*} )^2 } + C_2$ is provided.
Then,
as it holds that $(\ell_t^\top p_t - \ell_{ti^*})^2 \leq 1 - p_{ti^*}$,
we have \eqref{eq:aspsecond} with appropriate $A$ and $B$,
and consequently,
we obtain a regret bound given in \eqref{eq:resultsecond}.


\section{Regret lower bound}
\label{sec:lower}
This section provides (nearly) tight lower bounds for the expert problem and the MAB problem
in the adversarial regime with a self-bounding constraint.
We start with describing the statement for the expert problem:
\begin{theorem}
  \label{thm:lowerexpert}
  For any $\Delta \in (0, 1 / 4)$, $N \geq 4$, $T \geq 4 \log N$, $C \geq 0$,
  and for any algorithm for the expert problem,
  there exists an environment in the adversarial regime with a $(i^*, \Delta, N, C, T)$ self-bounding constraint
  for which
  the pseudo-regret is at least
  \begin{align}
    \label{eq:lowerexpert}
    \bar{R}_{T i^*} = \Omega \left( 
      \min \left\{
        \frac{\log N}{\Delta} + \sqrt{\frac{C \log N}{\Delta}},
        \sqrt{T \log N}
      \right\}
    \right) .
  \end{align}
\end{theorem}
To show this lower bound,
we define a distribution $\cD_{\Delta, i^*}$ over $\{ 0 , 1 \}^N$ for $\Delta \in (0, 1 / 4)$ and $i^* \in [N]$,
as follows:
if $\ell \sim \cD_{\Delta, i^*}$,
$\ell_{i^*}$ follows a Bernoulli distribution of parameter $1/2 - \Delta$,
i.e.,
$\prob [ \ell_{i^*} = 1 ] = 1/2 - \Delta$
and
$\prob [ \ell_{i^*} = 0 ] = 1/2 + \Delta$,
and 
$\ell_i$ follows a Bernoulli distribution of parameter $1 / 2$ for $i \in [N] \setminus i^*$,
independently.
We then can employ the following lemma:
\begin{lemma}[Proposition 2 in \citep{mourtada2019optimality}]
  \label{lem:degenne}
  For any algorithm and
  for any $\Delta \in (0, 1 / 4)$, $N \geq 4$ and $T \geq \frac{ \log N}{\Delta^2}$,
  there exists $i^*$ such that
    $
    \bar{R}_{T i^*} \geq \frac{\log N}{256 \Delta}
    $
  holds for $(\ell_t)_{t=1}^T \sim \cD_{\Delta, i^*}^T$.
\end{lemma}
Using this lower bound,
we can show Theorem~\ref{thm:lowerexpert}.\\
\textit{Proof of Theorem~\ref{thm:lowerexpert}.} \quad
  We show lower bounds for the following four cases: 
  (i) If $T \leq \frac{\log N}{\Delta^2}$,
  $\bar{R}_{Ti^*} = \Omega( \sqrt{T \log N})$.
  (ii) If $\frac{C}{\Delta} \leq \frac{\log N}{\Delta^2} \leq T$,
  $\bar{R}_{Ti^*} = \Omega( \frac{\log N}{\Delta})$.
  (iii) If $  \frac{\log N}{\Delta^2} \leq \frac{C}{\Delta}\leq T$,
  $\bar{R}_{Ti^*} = \Omega( \sqrt{\frac{C \log N}{\Delta}})$.
  (iv) If $  \frac{\log N}{\Delta^2} \leq T \leq \frac{C}{\Delta}$,
  $\bar{R}_{Ti^*} = \Omega( \sqrt{T \log N })$.
  Combining all four cases of (i)--(iv),
  we obtain \eqref{eq:lowerexpert}.

  (i) Suppose $T < \frac{\log N}{ \Delta^2}$.
  Set $\Delta' = \sqrt{\frac{\log N}{T}}$.
  We then have $T = \frac{\log N}{ \Delta'^2}$ and $\Delta < \Delta' \leq 1 / 4$.
  If $\ell_t \sim \cD_{\Delta', i^*}$ for all $t \in [T]$,
  then the environment is in an adversarial regime with a $(i^*, \Delta, N, C, T)$ self-bounding constraint for any $C\geq 0$,
  and the regret is bounded as
  $\bar{R}_{Ti^*} \geq \frac{\log N}{256 \Delta'} = \Omega( \sqrt{T \log N} )$ from Lemma~\ref{lem:degenne}.

  (ii) Suppose $\frac{C}{\Delta} \leq \frac{\log N}{\Delta^2} \leq T$.
  If $\ell_t \sim \cD_{\Delta, i^*}$ for all $t \in [T]$,
  the regret is bounded as
  $\bar{R}_{Ti^*} \geq \frac{\log N}{256 \Delta}$ for some $i^*$ from Lemma~\ref{lem:degenne}.
  The environment is in an adversarial regime with a $(i^*, \Delta, N, C, T)$ self-bounding for any $C\geq 0$.

  (iii) Suppose $\frac{\log N}{\Delta^2} \leq \frac{C}{\Delta}\leq T$.
  Define $\Delta' = \sqrt{\frac{\Delta \log N}{C}} \leq \Delta$.
  We then have
  $\frac{\log N}{\Delta'} = \sqrt{C \frac{\log N}{\Delta}}$.
  Let $T' = \lceil \frac{\log N}{ \Delta'^2 } \rceil = \lceil \frac{C}{\Delta} \rceil \leq T$.
  Consider an environment in which
  $\ell_t \sim \cD_{\Delta', i^*}$ for $t \in [T']$
  and
  $\ell_t \sim \cD_{\Delta, i^*}$ for $t \in [T'+1, T]$.
  Then from Lemma~\ref{lem:degenne},
  there exists $i^* \in [N]$ such that
  $\bar{R}_{Ti^*} \geq \bar{R}_{T'i^*} \geq \frac{\log N}{\Delta'} = \Omega(\sqrt{\frac{C \log N}{\Delta}})$.
  Further,
  we can show that the environment is in an adversarial regime with a $(i^*, \Delta, N, C, T)$ self-bounding constraint.
  In fact,
  we have $T' (\Delta - \Delta') \leq \frac{C}{\Delta} (\Delta - \Delta') \leq C$.

  (iv) Suppose $  \frac{\log N}{\Delta^2} \leq T \leq \frac{C}{\Delta}$.
  Set $\Delta' = \sqrt{\frac{\log N}{T}}$ and 
  consider $\ell_t \sim \cD_{\Delta', i^*}$ for all $t \in [T]$.
  Then the regret is bounded as
  $\bar{R}_{Ti^*} \geq \frac{\log N}{256 \Delta'} = \Omega( \sqrt{T \log N} )$ for some $i^*$,
  from Lemma~\ref{lem:degenne}.
  We can confirm that the environment is in an adversarial regime with a $(i^*, \Delta, N, C, T)$ self-bounding constraint,
  as we have
  $\Delta' T \leq \Delta T \leq C$,
  where the first and second inequalities follow from $\frac{\log N}{\Delta^2} \leq T$
  and
  $T \leq \frac{C}{\Delta}$,
  respectively.
  \qedwhite

Via a similar strategy to this proof,
we can show the regret lower bound for the MAB problem as well:
\begin{theorem}
  \label{thm:lowerMAB}
  For any $\Delta \in (0, 1 / 4)$, $N \geq 4$, $T \geq 4 \log N$, $C \geq 0$,
  and for any multi-armed bandit algorithm,
  there exists an environment in the adversarial regime with a $(i^*, \Delta, N, C, T)$ self-bounding constraint
  for which
  the pseudo-regret is at least
  \begin{align}
    \label{eq:lowerMAB}
    \bar{R}_{T i^*} = \Omega \left( 
      \min \left\{
        \frac{N}{\Delta} + \sqrt{\frac{C N }{\Delta}},
        \sqrt{N T}
      \right\}
    \right) .
  \end{align}
\end{theorem}
We can show this theorem by means of the following lemma:
\begin{lemma}[\citep{auer2002nonstochastic}]
  \label{lem:auer}
  For any multi-armed bandit algorithm and
  for any $\Delta \in (0, 1 / 4)$, $N \geq 4$ and $T \geq \frac{N}{ \Delta^2}$,
  there exists $i^*$ such that
  $
    \bar{R}_{T i^*} \geq \frac{N}{32 \Delta}
    $
  holds for $(\ell_t)_{t=1}^T \sim \cD_{\Delta, i^*}^T$.
\end{lemma}
A complete proof of Theorem~\ref{thm:lowerMAB} can be found in the appendix.

\section{Discussion}
In this paper, 
we have shown $O( R + \sqrt{CR} )$-regret bounds for the expert problem,
where $R$ stands for the regret bounds for the environment without corruptions
and $C$ stands for the corruption level.
From the matching lower bound,
we can see that this $O(\sqrt{CR})$-term characterizes the optimal robustness against the corruptions.
One natural question is whether such an $O(R + \sqrt{CR})$-type regret bounds can be found for other online decision problems,
such as online linear optimization,
online convex optimization,
linear bandits,
and convex bandits.
Other than the algorithms for the expert problem,
the Tsallis-INF algorithm by \citet{zimmert2021tsallis} for the MAB problem
is only concrete example that achieves $O(R + \sqrt{CR})$-regret
to our knowledge.
What these algorithm have in common is that
they use a framework of Follow-the-Regularized-Leader with decreasing learning rates
and that they achieve the best-of-both-world simultaneously.
As \citet{amir2020prediction} suggest,
Online Mirror Descent algorithms does not have $O(R + \sqrt{CR})$-regret bound,
in contrast to Follow-the-Regularized-Leader.
We believe that characterizing algorithms with $O(R + \sqrt{CR})$-regret bounds 
is an important future work.


\begin{ack}
    The author was supported by JST, ACT-I, Grant Number JPMJPR18U5, Japan.
\end{ack}

\bibliographystyle{abbrvnat}
\bibliography{reference}

\begin{thebibliography}{41}
\providecommand{\natexlab}[1]{#1}
\providecommand{\url}[1]{\texttt{#1}}
\expandafter\ifx\csname urlstyle\endcsname\relax
  \providecommand{\doi}[1]{doi: #1}\else
  \providecommand{\doi}{doi: \begingroup \urlstyle{rm}\Url}\fi

\bibitem[Amir et~al.(2020)Amir, Attias, Koren, Mansour, and
  Livni]{amir2020prediction}
I.~Amir, I.~Attias, T.~Koren, Y.~Mansour, and R.~Livni.
\newblock Prediction with corrupted expert advice.
\newblock \emph{Advances in Neural Information Processing Systems}, 33, 2020.

\bibitem[Arora et~al.(2012)Arora, Hazan, and Kale]{arora2012multiplicative}
S.~Arora, E.~Hazan, and S.~Kale.
\newblock The multiplicative weights update method: A meta-algorithm and
  applications.
\newblock \emph{Theory of Computing}, 8\penalty0 (1):\penalty0 121--164, 2012.

\bibitem[Audibert and Bubeck(2009)]{audibert2009minimax}
J.-Y. Audibert and S.~Bubeck.
\newblock Minimax policies for adversarial and stochastic bandits.
\newblock In \emph{Conference on Learning Theory}, pages 217--226, 2009.

\bibitem[Auer and Chiang(2016)]{auer2016algorithm}
P.~Auer and C.-K. Chiang.
\newblock An algorithm with nearly optimal pseudo-regret for both stochastic
  and adversarial bandits.
\newblock In \emph{Conference on Learning Theory}, pages 116--120. PMLR, 2016.

\bibitem[Auer et~al.(2002{\natexlab{a}})Auer, Cesa-Bianchi, and
  Fischer]{auer2002finite}
P.~Auer, N.~Cesa-Bianchi, and P.~Fischer.
\newblock Finite-time analysis of the multiarmed bandit problem.
\newblock \emph{Machine Learning}, 47\penalty0 (2-3):\penalty0 235--256,
  2002{\natexlab{a}}.

\bibitem[Auer et~al.(2002{\natexlab{b}})Auer, Cesa-Bianchi, Freund, and
  Schapire]{auer2002nonstochastic}
P.~Auer, N.~Cesa-Bianchi, Y.~Freund, and R.~E. Schapire.
\newblock The nonstochastic multiarmed bandit problem.
\newblock \emph{SIAM Journal on Computing}, 32\penalty0 (1):\penalty0 48--77,
  2002{\natexlab{b}}.

\bibitem[Bogunovic et~al.(2020)Bogunovic, Krause, and
  Scarlett]{bogunovic2020corruption}
I.~Bogunovic, A.~Krause, and J.~Scarlett.
\newblock Corruption-tolerant gaussian process bandit optimization.
\newblock In \emph{International Conference on Artificial Intelligence and
  Statistics}, pages 1071--1081. PMLR, 2020.

\bibitem[Bogunovic et~al.(2021)Bogunovic, Losalka, Krause, and
  Scarlett]{bogunovic2021stochastic}
I.~Bogunovic, A.~Losalka, A.~Krause, and J.~Scarlett.
\newblock Stochastic linear bandits robust to adversarial attacks.
\newblock In \emph{International Conference on Artificial Intelligence and
  Statistics}, pages 991--999. PMLR, 2021.

\bibitem[Bubeck and Slivkins(2012)]{bubeck2012best}
S.~Bubeck and A.~Slivkins.
\newblock The best of both worlds: Stochastic and adversarial bandits.
\newblock In \emph{Conference on Learning Theory}, pages 42.1--42.23, 2012.

\bibitem[Cesa-Bianchi and Lugosi(2006)]{cesa2006prediction}
N.~Cesa-Bianchi and G.~Lugosi.
\newblock \emph{Prediction, Learning, and Games}.
\newblock Cambridge University Press, 2006.

\bibitem[Cesa-Bianchi et~al.(2007)Cesa-Bianchi, Mansour, and
  Stoltz]{cesa2007improved}
N.~Cesa-Bianchi, Y.~Mansour, and G.~Stoltz.
\newblock Improved second-order bounds for prediction with expert advice.
\newblock \emph{Machine Learning}, 66\penalty0 (2):\penalty0 321--352, 2007.

\bibitem[De~Rooij et~al.(2014)De~Rooij, Van~Erven, Gr{\"u}nwald, and
  Koolen]{de2014follow}
S.~De~Rooij, T.~Van~Erven, P.~D. Gr{\"u}nwald, and W.~M. Koolen.
\newblock Follow the leader if you can, hedge if you must.
\newblock \emph{The Journal of Machine Learning Research}, 15\penalty0
  (1):\penalty0 1281--1316, 2014.

\bibitem[Degenne and Perchet(2016)]{degenne2016anytime}
R.~Degenne and V.~Perchet.
\newblock Anytime optimal algorithms in stochastic multi-armed bandits.
\newblock In \emph{International Conference on Machine Learning}, pages
  1587--1595. PMLR, 2016.

\bibitem[Freund and Schapire(1997)]{freund1997decision}
Y.~Freund and R.~E. Schapire.
\newblock A decision-theoretic generalization of on-line learning and an
  application to boosting.
\newblock \emph{Journal of Computer and System Sciences}, 55\penalty0
  (1):\penalty0 119--139, 1997.

\bibitem[Gaillard et~al.(2014)Gaillard, Stoltz, and
  Van~Erven]{gaillard2014second}
P.~Gaillard, G.~Stoltz, and T.~Van~Erven.
\newblock A second-order bound with excess losses.
\newblock In \emph{Conference on Learning Theory}, pages 176--196. PMLR, 2014.

\bibitem[Gupta et~al.(2019)Gupta, Koren, and Talwar]{gupta2019better}
A.~Gupta, T.~Koren, and K.~Talwar.
\newblock Better algorithms for stochastic bandits with adversarial
  corruptions.
\newblock In \emph{Conference on Learning Theory}, pages 1562--1578. PMLR,
  2019.

\bibitem[Hajiesmaili et~al.(2020)Hajiesmaili, Talebi, Lui, Wong,
  et~al.]{hajiesmaili2020adversarial}
M.~Hajiesmaili, M.~S. Talebi, J.~Lui, W.~S. Wong, et~al.
\newblock Adversarial bandits with corruptions: Regret lower bound and
  no-regret algorithm.
\newblock \emph{Advances in Neural Information Processing Systems}, 33, 2020.

\bibitem[Hazan and Kale(2010)]{hazan2010extracting}
E.~Hazan and S.~Kale.
\newblock Extracting certainty from uncertainty: Regret bounded by variation in
  costs.
\newblock \emph{Machine learning}, 80\penalty0 (2-3):\penalty0 165--188, 2010.

\bibitem[Huang et~al.(2016)Huang, Lattimore, Gy{\"o}rgy, and
  Szepesv{\'a}ri]{huang2016following}
R.~Huang, T.~Lattimore, A.~Gy{\"o}rgy, and C.~Szepesv{\'a}ri.
\newblock Following the leader and fast rates in linear prediction: Curved
  constraint sets and other regularities.
\newblock In \emph{Advances in Neural Information Processing Systems}, pages
  4970--4978, 2016.

\bibitem[Ito(2021)]{ito2021parameter}
S.~Ito.
\newblock Parameter-free multi-armed bandit algorithms with hybrid
  data-dependent regret bounds.
\newblock In \emph{Conference on Learning Theory}, pages 2552--2583. PMLR,
  2021.

\bibitem[Jun et~al.(2018)Jun, Li, Ma, and Zhu]{jun2018adversarial}
K.-S. Jun, L.~Li, Y.~Ma, and X.~Zhu.
\newblock Adversarial attacks on stochastic bandits.
\newblock In \emph{Proceedings of the 32nd International Conference on Neural
  Information Processing Systems}, pages 3644--3653, 2018.

\bibitem[Lai(1987)]{lai1987adaptive}
T.~L. Lai.
\newblock Adaptive treatment allocation and the multi-armed bandit problem.
\newblock \emph{Annals of Statistics}, 15\penalty0 (3):\penalty0 1091--1114,
  1987.

\bibitem[Lai and Robbins(1985)]{lai1985asymptotically}
T.~L. Lai and H.~Robbins.
\newblock Asymptotically efficient adaptive allocation rules.
\newblock \emph{Advances in Applied Mathematics}, 6\penalty0 (1):\penalty0
  4--22, 1985.

\bibitem[Lattimore and Szepesv{\'a}ri(2020)]{lattimore2020bandit}
T.~Lattimore and C.~Szepesv{\'a}ri.
\newblock \emph{Bandit Algorithms}.
\newblock Cambridge University Press, 2020.

\bibitem[Lee et~al.(2021)Lee, Luo, Wei, Zhang, and Zhang]{lee2021achieving}
C.-W. Lee, H.~Luo, C.-Y. Wei, M.~Zhang, and X.~Zhang.
\newblock Achieving near instance-optimality and minimax-optimality in
  stochastic and adversarial linear bandits simultaneously.
\newblock \emph{arXiv preprint arXiv:2102.05858}, 2021.

\bibitem[Li et~al.(2019)Li, Lou, and Shan]{li2019stochastic}
Y.~Li, E.~Y. Lou, and L.~Shan.
\newblock Stochastic linear optimization with adversarial corruption.
\newblock \emph{arXiv preprint arXiv:1909.02109}, 2019.

\bibitem[Littlestone and Warmuth(1994)]{littlestone1994weighted}
N.~Littlestone and M.~Warmuth.
\newblock The weighted majority algorithm.
\newblock \emph{Information and Computation}, 108\penalty0 (2):\penalty0
  212--261, 1994.

\bibitem[Liu and Shroff(2019)]{liu2019data}
F.~Liu and N.~Shroff.
\newblock Data poisoning attacks on stochastic bandits.
\newblock In \emph{International Conference on Machine Learning}, pages
  4042--4050. PMLR, 2019.

\bibitem[Luo and Schapire(2015)]{luo2015achieving}
H.~Luo and R.~E. Schapire.
\newblock Achieving all with no parameters: Adanormalhedge.
\newblock In \emph{Conference on Learning Theory}, pages 1286--1304. PMLR,
  2015.

\bibitem[Lykouris et~al.(2018)Lykouris, Mirrokni, and
  Paes~Leme]{lykouris2018stochastic}
T.~Lykouris, V.~Mirrokni, and R.~Paes~Leme.
\newblock Stochastic bandits robust to adversarial corruptions.
\newblock In \emph{Proceedings of the 50th Annual ACM SIGACT Symposium on
  Theory of Computing}, pages 114--122, 2018.

\bibitem[Lykouris et~al.(2019)Lykouris, Simchowitz, Slivkins, and
  Sun]{lykouris2019corruption}
T.~Lykouris, M.~Simchowitz, A.~Slivkins, and W.~Sun.
\newblock Corruption robust exploration in episodic reinforcement learning.
\newblock \emph{arXiv preprint arXiv:1911.08689}, 2019.

\bibitem[Masoudian and Seldin(2021)]{masoudian2021improved}
S.~Masoudian and Y.~Seldin.
\newblock Improved analysis of the tsallis-inf algorithm in stochastically
  constrained adversarial bandits and stochastic bandits with adversarial
  corruptions.
\newblock In \emph{Conference on Learning Theory}, pages 3330--3350. PMLR,
  2021.

\bibitem[Mourtada and Ga{\"\i}ffas(2019)]{mourtada2019optimality}
J.~Mourtada and S.~Ga{\"\i}ffas.
\newblock On the optimality of the hedge algorithm in the stochastic regime.
\newblock \emph{Journal of Machine Learning Research}, 20:\penalty0 1--28,
  2019.

\bibitem[Pogodin and Lattimore(2020)]{pogodin2020first}
R.~Pogodin and T.~Lattimore.
\newblock On first-order bounds, variance and gap-dependent bounds for
  adversarial bandits.
\newblock In \emph{Uncertainty in Artificial Intelligence}, pages 894--904,
  2020.

\bibitem[Sani et~al.(2014)Sani, Neu, and Lazaric]{sani2014exploiting}
A.~Sani, G.~Neu, and A.~Lazaric.
\newblock Exploiting easy data in online optimization.
\newblock In \emph{Advances in Neural Information Processing 27}, 2014.

\bibitem[Seldin and Lugosi(2017)]{seldin2017improved}
Y.~Seldin and G.~Lugosi.
\newblock An improved parametrization and analysis of the exp3++ algorithm for
  stochastic and adversarial bandits.
\newblock In \emph{Conference on Learning Theory}, pages 1743--1759, 2017.

\bibitem[Seldin and Slivkins(2014)]{seldin2014one}
Y.~Seldin and A.~Slivkins.
\newblock One practical algorithm for both stochastic and adversarial bandits.
\newblock In \emph{International Conference on Machine Learning}, pages
  1287--1295, 2014.

\bibitem[Steinhardt and Liang(2014)]{steinhardt2014adaptivity}
J.~Steinhardt and P.~Liang.
\newblock Adaptivity and optimism: An improved exponentiated gradient
  algorithm.
\newblock In \emph{International Conference on Machine Learning}, pages
  1593--1601. PMLR, 2014.

\bibitem[Wei and Luo(2018)]{wei2018more}
C.-Y. Wei and H.~Luo.
\newblock More adaptive algorithms for adversarial bandits.
\newblock In \emph{Conference On Learning Theory}, pages 1263--1291, 2018.

\bibitem[Zimmert and Seldin(2021)]{zimmert2021tsallis}
J.~Zimmert and Y.~Seldin.
\newblock Tsallis-inf: An optimal algorithm for stochastic and adversarial
  bandits.
\newblock \emph{Journal of Machine Learning Research}, 22\penalty0
  (28):\penalty0 1--49, 2021.

\bibitem[Zimmert et~al.(2019)Zimmert, Luo, and Wei]{zimmert2019beating}
J.~Zimmert, H.~Luo, and C.-Y. Wei.
\newblock Beating stochastic and adversarial semi-bandits optimally and
  simultaneously.
\newblock In \emph{International Conference on Machine Learning}, pages
  7683--7692. PMLR, 2019.

\end{thebibliography}

\appendix

\newpage
\section{Appendix}
\subsection{Proof of Lemma~\ref{lem:boundMWU}}
The Hedge algorithm defined in \eqref{eq:defMWU}
can be interpreted as a special case of follow-the-regularized leader (FTRL) methods,
as follows:
\begin{align}
    p_t
    \in
    \argmin_{
        p \in [0, 1]^N:
        \| p \|_1 = 1
    }
    \left\{
        \sum_{j=1}^{t-1}
        \ell_{j}^\top p
        -
        \frac{1}{\eta_t}
        H(p)
    \right\}.
\end{align}
From a standard analysis of FTRL (see,
e.g.,
Exercise 28.12 in the book by \citet{lattimore2020bandit},
where we set $F_t(x) = - \frac{1}{\eta_t}H(x)$
),
we have
\begin{align}
    \label{eq:boundMWU1}
    R_{T}
    \leq
    \sum_{t=1}^T
    \left(
        \ell_t^\top (p_t - p_{t+1})
        -
        \frac{1}{\eta_{t}}
        KL(p_{t+1}, p_t)
    \right)
    +
    \frac{H (p_1) }{\eta_1}
    +
    \sum_{t=1}^T
    \left(
        \frac{1}{\eta_{t+1}}
        -
        \frac{1}{\eta_{t}}
    \right)
    H(p_{t+1}),
\end{align}
where $KL(p, q)$ represents the KL divergence
defined by
$
KL(p, q)
=
\sum_{i=1}^N
(p_i \log \frac{p_i}{q_i} 
-
p_i
+
q_i
)
$.
For any $\alpha_t \in \re$,
the first term in the right-hand side can be bounded as
\begin{align}
    \nonumber
    \ell_t^\top (p_t - p_{t+1})
    -
    \frac{1}{\eta_{t}}
    KL(p_{t+1}, p_t)
    &
    =
    (\ell_t - \alpha_t \mathbf{1})^\top (p_t - p_{t+1})
    -
    \frac{1}{\eta_{t}}
    KL(p_{t+1}, p_t)
    \\
    &
    \leq
    \max_{p \in \re_{>0}^N}
    \left\{
    (\ell_t - \alpha_t \mathbf{1})^\top (p_t - p)
    -
    \frac{1}{\eta_{t}}
    KL(p, p_t)
    \right\},
    \label{eq:boundKL}
\end{align}
where $\mathbf{1} \in \re^N$ represents the all-one vector,
and the equality follows from the fact that
$
\sum_{i=1}^N p_{ti}=
\sum_{i=1}^N p_{t+1,i}=
1
$.
The maximum in \eqref{eq:boundKL} is attained by
$p=( p_{ti} \exp( - \eta_t (\ell_{ti} - \alpha_t) ) )_{i=1}^N$.
In fact,
the objective function is concave in $p$ and
its gradient can be expressed as
\begin{align*}
    \nabla_p
    \left(
    (\ell_t - \alpha_t \mathbf{1})^\top (p_t - p)
    -
    \frac{1}{\eta_{t}}
    KL(p, p_t)
    \right)
    =
    -
    (\ell_t - \alpha_t \mathbf{1})
    -
    \frac{1}{\eta_t}
    \left(
        (\log p_i)_{i=1}^N
        -
        (\log p_{ti})_{i=1}^N
    \right),
\end{align*}
which is the zero vector
if and only if
$p=( p_{ti} \exp( - \eta_t (\ell_{ti} - \alpha_t) ) )_{i=1}^N$.
By substituting this into the objective function,
we have
\begin{align*}
    &
    \max_{p \in \re_{>0}^d}
    \left\{
    (\ell_t - \alpha_t \mathbf{1})^\top (p_t - p)
    -
    \frac{1}{\eta_{t}}
    KL(p, p_t)
    \right\}
    \\
    &
    =
    \sum_{i=1}^N
    \left(
    (\ell_{ti} - \alpha_t)
    p_{ti}
    -
    \frac{1}{\eta_t}
    (p_{ti} - 
    p_{ti} \exp(- \eta_t (\ell_{ti} - \alpha_t)) )
    \right)
    \\
    &
    =
    \frac{1}{\eta_t}
    \sum_{i=1}^N
    p_{ti}
    \left(
    \exp(- \eta_t (\ell_{ti} - \alpha_t)) )
    + \eta_t(\ell_{ti} - \alpha_t)
    - 1
    \right)
    =
    \frac{1}{\eta_t}
    \sum_{i=1}^N
    g\left( 
    - \eta_t(\ell_{ti} - \alpha_t)
    \right) .
\end{align*}
Combining this with \eqref{eq:boundMWU1} and \eqref{eq:boundKL},
and from $H(p_1) = H (\mathbf{1} / N) = \log N$,
we obtain \eqref{eq:boundMWU0}.
\qed

\subsection{Proof of (\ref{eq:decreasing1})}
To show \eqref{eq:decreasing1},
we use the following upper bound on $H(p)$:
\begin{lemma}
    For any $p \in [0, 1]^N$
    such that $\| p \|_1 = 1$
    and $i^* \in [N]$,
    we have
    \begin{align}
        \label{eq:boundH}
        H(p)
        \leq 
        (1 - p_i^*)
        \left( 1 + \log \frac{N-1}{1 - p_{i^*}} \right).
    \end{align}
\end{lemma}
\begin{proof}
    The value of $H(p)$ can be expressed as
    \begin{align}
        \label{eq:boundH0}
        H(p)
        =
        p_{i^*} \log \frac{1}{p_{i^*}}
        +
        \sum_{i \in [N] \setminus \{ i^* \}}
        p_i \log \frac{1}{p_i} .
    \end{align}
    The first term can be bounded as
    \begin{align}
        \label{eq:boundH1}
        p_{i^*} \log \frac{1}{p_{i^*}}
        =
        p_{i^*}\log \left(
            1 + \frac{1 - p_{i^*}}{p_{i^*}}
        \right)
        \leq
        p_{i^*} \left(
            \frac{1 - p_{i^*}}{p_{i^*}}
        \right)
        =
        1 - p_{i^*},
    \end{align}
    where the inequality follows from $\log (1+x) \leq x$ that holds for any $x > -1$.
    When $p_{i^*}$ is fixed,
    the second term of the right-hand side of \eqref{eq:boundH0} is
    maximized by setting $p_i = \frac{1-p_{i^*}}{N-1}$ for all $i \in [N] \setminus \{ i^* \}$,
    and hence,
    its value can be bounded as
    \begin{align}
        \sum_{i \neq [N] \setminus \{ i^* \}}
        p_i \log \frac{1}{p_i} 
        \leq
        (N-1)
        \frac{1- p_{i^*}}
        {N-1}
        \log \frac{N-1}{1 - p_{i^*}}
        =
        (1- p_{i^*})
        \log \frac{N-1}{1 - p_{i^*}} .
    \end{align}
    Combining this with \eqref{eq:boundH0} and \eqref{eq:boundH1},
    we obtain \eqref{eq:boundH}.
\end{proof}
\textit{Proof of \eqref{eq:decreasing1}.}~
  Set $T' = \lfloor 8 \log N \rfloor$.
  We then have $\eta_t \leq 1$ for $t > T'$.
  From Lemma~\ref{lem:boundMWU}
  with $\alpha_t = 0$ for $t \leq T'$ and $\alpha_t = \ell_{ti^*}$ for $t > T$,
  we have
  \begin{align*}
    &
    R_{Ti^*}
    \leq
    \frac{\log N}{\eta_1}
    +
    \sum_{t=1}^{T'}
    \frac{1}{\eta_t}
    \sum_{i=1}^N
    p_{ti} g(\eta_t(-\ell_{ti}))
    +
    \sum_{t=T'+1}^T
    \frac{1}{\eta_t}
    \sum_{i=1}^N
    p_{ti} g(\eta_t(\ell_{ti^*} - \ell_{ti}))
    +
    \sum_{t=1}^T
    \left(
        \frac{1}{\eta_{t+1}}
        -
        \frac{1}{\eta_{t}}
    \right)
    H(p_{t+1})
    \\
    &
    \leq
    \frac{\log N}{\eta_1}
    +
    \sum_{t=1}^{T'}
    \sum_{i=1}^N
    p_{ti}
    +
    (e-2)
    \sum_{t=T'+1}^T
    \eta_t
    \sum_{i=1}^N
    p_{ti} (\ell_{ti} - \ell_{ti^*})^2
    +
    \sum_{t=1}^T
    \left(
        \frac{1}{\eta_{t+1}}
        -
        \frac{1}{\eta_{t}}
    \right)
    H(p_{t+1})
    \\
    &
    \leq
    \frac{\log N}{\eta_1}
    +
    T'
    +
    (e-2)
    \sum_{t=T'+1}^T
    \eta_t
    (1-p_{ti^*})
    +
    \sum_{t=1}^T
    \left(
        \frac{1}{\eta_{t+1}}
        -
        \frac{1}{\eta_{t}}
    \right)
    (1 - p_{ti^*})
    \left( 1 + \log \frac{N-1}{1 - p_{ti^*}} \right)
    \\
    &
    \leq
    \sqrt{\frac{\log N}{8}}
    +
    8 \log N
    +
    \sqrt{8\log N}
    (e-2)
    \sum_{t=T'+1}^T
    \frac{1-p_{ti^*}}{\sqrt{t}}
    +
    \frac{1}{2\sqrt{8\log N}}
    \sum_{t=1}^T
    \frac{1 - p_{ti^*}}{\sqrt{t}}
    \left( 1 + \log \frac{N-1}{1 - p_{ti^*}} \right)
    \\&
    \leq
    32 \log N
    +
    \sqrt{8 \log N}
    \sum_{t=1}^{T}
    \frac{1 - p_{ti^*}}{\sqrt{t}}
    \left(
      1
      +
      \frac{1}{16 \log N } 
      \log \frac{1}{1 - p_{ti^*}}
    \right),
  \end{align*}
  where the second inequality follows from
  $g(-y) \leq y$ for $y\geq 0$,
  $g(-y) \leq (\mathrm{e} - 2)y^2$ for $y \geq - 1$,
  and $\ell_{ti} \in [0, 1]$.
  The third inequality follows from \eqref{eq:boundH}
  and the forth inequality follows from $\eta_{t} = \sqrt{\frac{8\log N}{t}}$.
  \qed

\subsection{Proof of (\ref{eq:decreasing2})}
We use the following lemma
to bound the left-hand side of \eqref{eq:decreasing2}.
\begin{lemma}
    For any $c \in \re$,
    we have
  \begin{align}
    \label{eq:maxfx}
    \max_{x \in (0, 1]}
    x (c - \log x) \leq
    \left\{
      \begin{array}{ll}
        \exp(c-1) & c \leq 1
        \\
        c & c > 1
      \end{array}
    \right. .
  \end{align}
\end{lemma}
\begin{proof}
    Define $f:\re_{>0} \rightarrow \re$ by
    $f(x) = x (c - \log x)$.
    The derivative
    of $f$ can be expressed as
  \begin{align}
    f'(x)
    =
    (c - \log x) - 1 ,
  \end{align}
  which implies that $x = \exp(c-1)$ is the unique stationary point of $f$.
  As $f(x)$ is concave in $x \in \re_{>0}$,
  subject to the constraint of $x \in (0, 1]$,
  $f(x)$ is maximized by
$
x = \min \{1, \exp(c - 1) \}
$.
By substituting this into
$f(x) = x (c - \log x)$,
we obtain \eqref{eq:maxfx}.
\end{proof}
\textit{Proof of \eqref{eq:decreasing2}.}~
  Set $T' = \lceil (\frac{a}{b})^2 \rceil$.
  From \eqref{eq:maxfx},
  we have
  \begin{align}
    \nonumber
    &
    \sum_{t=1}^T
    \frac{x_t}{\sqrt{t}}
    \left( a - b \sqrt{t} - \log x_t \right)
    =
    \sum_{t=1}^{T'}
    \frac{x_t}{\sqrt{t}}
    \left( a - b \sqrt{t} - \log x_t \right)
    +
    \sum_{t=T' + 1}^{T}
    \frac{x_t}{\sqrt{t}}
    \left( a - b \sqrt{t} - \log x_t \right)
    \\
    &
    \leq
    \sum_{t=1}^{T'}
    \frac{a - b\sqrt{t}}{\sqrt{t}}
    +
    \sum_{t=T' + 1}^{\infty}
    \frac{1}{\sqrt{t}}
    \exp \left(
      a - b \sqrt{t} - 1
    \right)
    \leq
    a
    \sum_{t=1}^{T'}
    \frac{1}{\sqrt{t}}
    +
    \exp(a - 1)
    \sum_{t=T' + 1}^{\infty}
    \frac{
    \exp \left(
      - b \sqrt{t} 
    \right)
    }{\sqrt{t}}.
    \label{eq:decreasing20}
  \end{align}
  The first term can be bounded by
  \begin{align}
    \label{eq:boundsumsqrt}
    \sum_{t=1}^{T'}
    \frac{1}{\sqrt{t}}
    \leq
    2
    \sum_{t=1}^{T'}
      \frac{1}{\sqrt{t} + \sqrt{t - 1}}
    =
    2
    \sum_{t=1}^{T'}
    \left(
      \sqrt{t} - \sqrt{t-1}
    \right)
    =
    2 \sqrt{T'}
    \leq
    2 \sqrt{
        \left(
            \frac{a}{b}
        \right)^2
        + 1
    }
    \leq
    \frac{2 a}{b} + \frac{b}{a}.
  \end{align}
  Further,
  as $\exp(- b \sqrt{y})$ is convex in $y \geq 0$,
  and as its derivative in $y$ can be expressed as
  $- \frac{b}{2\sqrt{y}} \exp(- b \sqrt{y})$,
  we have
  \begin{align}
    \label{eq:boundsqrtexp}
    \exp(- b \sqrt{t - 1})
    -
    \exp(- b \sqrt{t})
    \geq
    \frac{b}{2 \sqrt{t}}
    \exp(- b \sqrt{t}).
  \end{align}
  From this,
  we have
  \begin{align*}
    \exp(a-1)
    \sum_{t=T' + 1}^{\infty}
    \frac{
    \exp \left(
      - b \sqrt{t} 
    \right)
        }{\sqrt{t}}
    &
    \leq
    \frac{2
    \exp(a-1)
    }{b}
    \sum_{t=T'+1}^{\infty}
    \left(
    \exp(- b \sqrt{t - 1})
    -
    \exp(- b \sqrt{t})
    \right)
    \\
    &
    =
    \frac{2}{b}
    \exp( - b \sqrt{T'} + a - 1)
    \leq
    \frac{2}{b}
    \exp(-1)
    \leq
    \frac{1}{b},
  \end{align*}
  where the first inequality follows from \eqref{eq:boundsqrtexp} and
  the second inequality follows from $T' \geq (\frac{a}{b})^2$.
  Combining this with
  \eqref{eq:decreasing20} and \eqref{eq:boundsumsqrt},
  we obtain \eqref{eq:decreasing2}.
  \qed

\subsection{Proof of Theorem~\ref{thm:lowerMAB}}
  We show lower bounds for the following four cases: 
  (i) If $T \leq \frac{N}{\Delta^2}$,
  $\bar{R}_{Ti^*} = \Omega( \sqrt{T N})$.
  (ii) If $\frac{C}{\Delta} \leq \frac{N}{\Delta^2} \leq T$,
  $\bar{R}_{Ti^*} = \Omega( \frac{N}{\Delta})$.
  (iii) If $  \frac{N}{\Delta^2} \leq \frac{C}{\Delta}\leq T$,
  $\bar{R}_{Ti^*} = \Omega( \sqrt{\frac{C N}{\Delta}})$.
  (iv) If $  \frac{ N}{\Delta^2} \leq T \leq \frac{C}{\Delta}$,
  $\bar{R}_{Ti^*} = \Omega( \sqrt{TN})$.
  Combining all four cases of (i)--(iv),
  we obtain \eqref{eq:lowerMAB}.

  (i) Suppose $T < \frac{ N}{ \Delta^2}$.
  Set $\Delta' = \sqrt{\frac{ N}{T}}$.
  We then have $T = \frac{ N}{ \Delta'^2}$ and $\Delta < \Delta' \leq 1 / 4$.
  If $\ell_t \sim \cD_{\Delta', i^*}$ for all $t \in [T]$,
  then the environment is in an adversarial regime with a $(i^*, \Delta, N, C, T)$ self-bounding constraint for any $C\geq 0$,
  and the regret is bounded as
  $\bar{R}_{Ti^*} \geq \frac{N}{32 \Delta'} = \Omega( \sqrt{T  N} )$ from Lemma~\ref{lem:auer}.

  (ii) Suppose $\frac{C}{\Delta} \leq \frac{N}{\Delta^2} \leq T$.
  If $\ell_t \sim \cD_{\Delta, i^*}$ for all $t \in [T]$,
  the regret is bounded as
  $\bar{R}_{Ti^*} \geq \frac{N}{32 \Delta}$ for some $i^*$ from Lemma~\ref{lem:auer}.
  The environment is in an adversarial regime with a $(i^*, \Delta, N, C, T)$ self-bounding constraint for any $C\geq 0$.

  (iii) Suppose $\frac{N}{\Delta^2} \leq \frac{C}{\Delta}\leq T$.
  Define $\Delta' = \sqrt{\frac{\Delta  N}{C}} \leq \Delta$.
  We then have
  $\frac{ N}{\Delta'} = \sqrt{C \frac{ N}{\Delta}}$.
  Let $T' = \lceil \frac{ N}{ \Delta'^2 } \rceil = \lceil \frac{C}{\Delta} \rceil \leq T$.
  Consider an environment in which
  $\ell_t \sim \cD_{\Delta', i^*}$ for $t \in [T']$
  and
  $\ell_t \sim \cD_{\Delta, i^*}$ for $t \in [T'+1, T]$.
  Then from Lemma~\ref{lem:auer},
  there exists $i^* \in [N]$ such that
  $\bar{R}_{Ti^*} \geq \bar{R}_{T'i^*} \geq \frac{N}{32 \Delta'} = \Omega(\sqrt{\frac{C N}{\Delta}})$.
  Further,
  we can show that the environment is in an adversarial regime with a $(i^*, \Delta, N, C, T)$ self-bounding constraint.
  In fact,
  we have $T' (\Delta - \Delta') \leq \frac{C}{\Delta} (\Delta - \Delta') \leq C$.

  (iv) Suppose $  \frac{ N}{\Delta^2} \leq T \leq \frac{C}{\Delta}$.
  Set $\Delta' = \sqrt{\frac{ N}{T}}$ and 
  consider $\ell_t \sim \cD_{\Delta', i^*}$ for all $t \in [T]$.
  Then the regret is bounded as
  $\bar{R}_{Ti^*} \geq \frac{ N}{32 \Delta'} = \Omega( \sqrt{T  N} )$ for some $i^*$,
  from Lemma~\ref{lem:auer}.
  We can confirm that the environment is in an adversarial regime with a $(i^*, \Delta, N, C, T)$ self-bounding constraint,
  as we have
  $\Delta' T \leq \Delta T \leq C$,
  where the first and second inequalities follow from $\frac{N}{\Delta^2} \leq T$
  and
  $T \leq \frac{C}{\Delta}$,
  respectively.
  \qed

\end{document}